\documentclass[12pt]{article}
\usepackage{amsmath, amssymb, amsfonts}
\usepackage{amsthm}
\usepackage{geometry}
\usepackage{tikz}
\usetikzlibrary{positioning,calc,arrows.meta}
\usepackage{cite}
\usepackage[colorlinks=true,linkcolor=blue,citecolor=blue,urlcolor=blue]{hyperref}

\newtheorem{theorem}{Theorem}

\title{Universal Approximation Theorem for Input-Connected Multilayer Perceptrons}
\author{Vugar Ismailov\thanks{The author can be contacted at
\texttt{vugaris@mail.ru} or \texttt{vugaris@gmail.com}.}}

\date{\today}

\begin{document}

\maketitle

\begin{abstract}
We present the \emph{Input-Connected Multilayer Perceptron} (IC--MLP), a
feedforward neural network architecture in which each hidden neuron receives,
in addition to the outputs of the preceding layer, a direct affine connection
from the raw input. We first study this architecture in the univariate setting
and give an explicit and systematic description of IC--MLPs with an arbitrary
finite number of hidden layers, including iterated formulas for the network functions. 
In this setting, we prove a universal approximation theorem showing that deep IC--MLPs 
can approximate any continuous function on a closed interval of the real line if
and only if the activation function is nonlinear. We then extend the analysis
to vector-valued inputs and establish a corresponding universal approximation
theorem for continuous functions on compact subsets of $\mathbb{R}^n$.

\end{abstract}

\section{Introduction}

Multilayer perceptrons (MLPs) are widely used for function approximation,
regression, classification, and pattern recognition. An MLP consists of an
input layer, one or more hidden layers, and an output layer. Each neuron forms
a weighted sum of its inputs, adds a bias, and then applies an activation
function. From the approximation-theoretic point of view, an MLP generates a
family of parametrized functions, and the main question is how large this
family is.

A basic notion in this area is the \emph{universal approximation property}.
Roughly speaking, a class of networks has this property if it can approximate
every continuous function on a compact set with arbitrary accuracy. The first
results of this type were obtained for \emph{shallow} networks, i.e., networks
with a single hidden layer. For an input $\mathbf{x}\in\mathbb{R}^d$, a typical
single-hidden-layer network computes a function of the form
\begin{equation}\label{eq:shallow}
\sum_{i=1}^{r} c_i\,\sigma(\mathbf{w}_i\cdot\mathbf{x}+b_i),
\end{equation}
where $r\in\mathbb{N}$, $c_i,b_i\in\mathbb{R}$, $\mathbf{w}_i\in\mathbb{R}^d$,
and $\sigma:\mathbb{R}\to\mathbb{R}$ is the activation function. The
approximation question asks when the set of all functions of the form
\eqref{eq:shallow} is dense in $C(K)$ for every compact set $K\subset\mathbb{R}^d$.

\medskip

\noindent
\textbf{Shallow networks and universal approximation.}
The possibility of approximating continuous functions on compact subsets of
$\mathbb{R}^d$ by networks of the form \eqref{eq:shallow} has been studied in a
large number of works. Different approaches were used, including methods from
harmonic analysis, functional analysis, and classical approximation theory.

Carroll and Dickinson \cite{CarrollDickinson1989} used ideas related to the
inverse Radon transform to establish universality results for certain
single-hidden-layer neural networks. Gallant and White
\cite{GallantWhite1988} constructed a specific continuous, nondecreasing
\emph{sigmoidal} activation function (referred to as a ``cosine squasher'') and
showed that it yields density via Fourier series expansions. Here, a sigmoidal
function is understood as a function $\sigma:\mathbb{R}\to\mathbb{R}$ that
tends to $0$ as $x\to -\infty$ and to $1$ as $x\to +\infty$.

Cybenko \cite{Cybenko1989} and Funahashi \cite{Funahashi1989} independently
proved that feedforward networks with a continuous sigmoidal activation
function can approximate any continuous function on compact subsets of
$\mathbb{R}^d$. Cybenko's proof relies on methods from functional analysis,
particularly the Hahn--Banach theorem and the Riesz representation theorem, while
Funahashi's approach is based on an integral representation due to Irie and
Miyake \cite{IrieMiyake1988}, using a kernel that can be expressed as a
difference of two sigmoidal functions.

Hornik, Stinchcombe and White \cite{HornikStinchcombeWhite1989} provided another
general universality theorem for feedforward networks, based on an application
of the Stone--Weierstrass theorem. See also Cotter \cite{Cotter1990} for a
discussion of this theorem and its role in neural network approximation.

It is also important that universality can be obtained under various
restrictions and in constructive forms. K\r{u}rkov\'{a} \cite{Kurkova1992}
showed that staircase-like functions of sigmoidal type can approximate
continuous functions on compact intervals in $\mathbb{R}$, and used this idea
in related results for deeper networks (see
\cite{Kurkova1991,Kurkova1992}). Chen, Chen and Liu
\cite{ChenChenLiu1992} extended Cybenko's theorem by allowing bounded (not
necessarily continuous) sigmoidal activation functions. A closely related
result was obtained independently by Jones \cite{Jones1990}. 
Other works investigated the effect of restricting weights and thresholds: 
Chui and Li \cite{ChuiLi1992} treated integer weights and thresholds, while Ito
\cite{Ito1991} studied density results with unit weights. Constructive 
approximation schemes for special sigmoidal superpositions with a
minimal number of terms, as well as their application to single- and
two-hidden-layer networks with fixed weights, were developed in joint work
with Guliyev \cite{GuliyevIsmailov2016,GuliyevIsmailov2018a,GuliyevIsmailov2018b}.
A broader discussion of such restrictions, together with a comprehensive survey 
of related results, can be found in Pinkus \cite{Pinkus1999} 
and in our monograph \cite[Chapter~5]{Ismailov2021}.

A natural direction, beyond sigmoidal activations, is to ask for general
conditions on $\sigma$ guaranteeing density. Many works established universality
results for non-sigmoidal activation functions; see, for example,
Stinchcombe and White \cite{StinchcombeWhite1990}, Hornik \cite{Hornik1991}, and
Mhaskar and Micchelli \cite{MhaskarMicchelli1992}. A complete characterization
was obtained by Leshno, Lin, Pinkus and Schocken
\cite{LeshnoLinPinkusSchocken1993}: for a continuous activation function, the
family \eqref{eq:shallow} is dense in $C(K)$ for every compact
$K \subset \mathbb{R}^d$ if and only if $\sigma$ is not a polynomial. Thus, for
shallow networks, the universality question is, in a strong sense, completely
understood.

\medskip

\noindent
\textbf{From shallow to deep networks.}
While the classical theory focuses on \eqref{eq:shallow}, modern applications
often use several hidden layers. A standard two-hidden-layer feedforward
network computes a function of the form
\[
\sum_{j=1}^{N_2} v_j\,
\sigma\!\left(\sum_{i=1}^{N_1} w_{ji}\,\sigma(\mathbf{a}_i\cdot\mathbf{x}+b_i)
+ d_j\right),
\]
where $N_1,N_2\in\mathbb{N}$ and the remaining quantities are real parameters.
Deeper networks are obtained by iterating this construction: each additional
hidden layer applies an affine map to the outputs of the preceding layer and
passes the result through $\sigma$.

There are two related but distinct issues in the study of deep networks. The
first is \emph{universality}: does a given deep architecture approximate every
continuous function on compact sets? The second concerns \emph{efficiency}:
how the approximation error depends on the depth and width, or how many
parameters are needed to reach a given accuracy. The present paper is devoted
to the first issue, namely universality, and we briefly recall some
developments in this direction.

Universality can be achieved in many deep-network settings under explicit
assumptions on the activation function. To the best of our knowledge,
Gripenberg \cite{Gripenberg2003} was the first to study deep networks with a
bounded number of neurons in each hidden layer and to prove universality under
the assumption that the activation function is continuous, nonaffine, and
twice continuously differentiable in a neighborhood of a point $t$ with
$\sigma''(t)\neq 0$.

Kidger and Lyons \cite{KidgerLyons2020} later weakened this smoothness
requirement. They showed that universality for deep and narrow networks still
holds provided the activation function is continuous, nonaffine, and
differentiable in a neighborhood of at least one point at which the
derivative is continuous and nonzero. Thus, in both \cite{Gripenberg2003}
and \cite{KidgerLyons2020}, universality is obtained under local
differentiability assumptions on the activation function, without requiring
any global smoothness.

For piecewise linear activations, Hanin and Sellke~\cite{HaninSellke2018}
obtained a sharp characterization of universality for deep ReLU networks.
They proved that, for input dimension $d$, a ReLU network with arbitrary depth
can approximate every continuous function on compact subsets of $\mathbb{R}^d$
if and only if every hidden layer has width at least $d+1$ (that is, contains
at least $d+1$ neurons).

In a different but related direction, Johnson~\cite{Johnson2019} established a
purely negative universality result for a broad class of activation functions.
He proved that if the activation function is uniformly continuous and can be
approximated by a sequence of one-to-one functions, then layered feedforward
networks in which each hidden layer has width at most $d$ cannot approximate
all continuous functions on $\mathbb{R}^d$, regardless of depth. This result
does not assert that universality holds for width $d+1$ or larger; rather, it
shows that width strictly greater than $d$ is necessary for universal
approximation within this class of activation functions.

We refer the reader to Goodfellow, Bengio and Courville
\cite{GoodfellowBengioCourville2016} for background on neural network
architectures and motivations from machine learning, and to the recent
monograph by Petersen and Zech \cite{PetersenZech2024} for a mathematical
analysis of deep neural networks, with an emphasis on approximation theory,
optimization theory, and statistical learning theory.

\medskip

\noindent
\textbf{Motivation for input-connected architectures.}
In standard feedforward MLPs, the raw input $\mathbf{x}$ is used only in the
first hidden layer. After that, deeper layers receive information only through
the outputs of preceding layers. In this paper we present a simple and
natural modification: we allow each hidden neuron to access the raw input
directly, in addition to the outputs of the previous layer. This creates an
\emph{input-connected} architecture.

We call the resulting model the \emph{Input-Connected Multilayer Perceptron}
(IC--MLP). Formally, in an IC--MLP, every hidden neuron computes an expression
of the form
\[
\sigma\!\big(\text{(affine combination of previous-layer outputs)}
\;+\; \text{(affine function of the raw input)}\big),
\]
so that information from the input can enter at every depth. Standard MLPs are
recovered as a special case by suppressing these direct input connections
beyond the first hidden layer. In this sense, IC--MLPs form a strict extension
of the classical feedforward architecture, while preserving the basic
feedforward structure.

The main point of this paper is that this architecture admits a simple and
transparent universality theory. In particular, we obtain a universal 
approximation theorem under a minimal condition on the activation function.

\medskip

\noindent
\textbf{Overview of results.}
Our analysis proceeds in two stages. We first consider the univariate setting,
where the input is a scalar $x\in\mathbb{R}$. This case already captures the
structural feature that distinguishes IC--MLPs from classical MLPs, and it
allows a systematic description of networks of arbitrary finite depth. 
For completeness, we describe the network functions for one, two,
and three hidden layers, and then pass to the general $L$-hidden-layer
architecture in a unified notation.

In the univariate setting, we prove a universal approximation theorem that gives 
a necessary and sufficient condition: deep IC--MLPs approximate every continuous
function on a compact interval of the real line if and only if the activation
function is nonlinear (that is, not of the form $\sigma(t)=at+b$). This result parallels 
the classical ``if and only if'' characterizations known for shallow networks
(see, for instance, \cite{LeshnoLinPinkusSchocken1993}), but is established here
for the input-connected deep architecture under a strictly weaker condition on the activation function.

We then extend the analysis to vector inputs $\mathbf{x}\in\mathbb{R}^n$ and
show that the same equivalence holds for the approximation of continuous functions
on compact subsets of $\mathbb{R}^n$. Thus, IC--MLPs admit a universal
approximation theorem in the multivariate setting as well.

\medskip

\noindent
\textbf{Relation to existing architectures.}
Architectures that allow the raw input to be fed directly to deeper layers have 
appeared previously in different contexts. Most notably, Input Convex Neural Networks (ICNNs) 
\cite{AmosXuKolter2017} introduce so-called \emph{passthrough} connections that link the 
input directly to hidden units in deeper layers. In that setting, such connections are 
introduced to ensure that the network output is convex with respect to selected inputs. 
This convexity is enforced by requiring certain inter-layer weights to be nonnegative, 
and the direct passthrough connections compensate for the resulting restrictions.

Related ideas of bypassing intermediate layers have also been explored in other deep 
learning architectures, such as residual networks \cite{HeZhangRenSun2016} and densely 
connected convolutional networks \cite{HuangLiuMaatenWeinberger2017}. In those models, 
skip connections transmit information from earlier hidden layers to later hidden layers, 
while the raw input appears only at the first layer and does not enter deeper layers 
through independent affine maps.

Constructions involving so-called \emph{source channels} that push
the input values forward appear in the approximation-theoretic analysis
of deep ReLU networks \cite{DaubechiesDeVoreFoucartHaninPetrova2022}.
In that work, the authors introduce \emph{special networks} equipped with a
source channel (SC) and a collation channel (CC). The source channel
carries the input unchanged across all layers, ensuring that it remains
available at each stage of the construction, while the collation
channel aggregates the value of certain intermediate computations. 
These special networks are not standard feedforward networks, since neurons
in the source and collation channels may bypass the activation; they serve as an auxiliary
architecture used in proofs of approximation results and in deriving
important properties of ReLU networks.

The IC--MLP architecture studied in this paper differs from the above
approaches in several respects. It incorporates direct affine dependence
on the input at every layer as part of the network definition, rather
than propagating the input unchanged through a separate channel or
restricting its use to specific constructions. No convexity constraints
are imposed, and the networks are evaluated in a standard feedforward
manner. Moreover, the model allows for general activation functions,
not only ReLU. In this paper, we study it in the context of
universal approximation.

\section{Scalar IC--MLP Architecture}

In this section we describe IC--MLP
architecture in the special case of a \emph{scalar input}. Throughout this
section, the input variable is a real number $x\in\mathbb{R}$, and
$\sigma:\mathbb{R}\to\mathbb{R}$ denotes a fixed continuous activation function.
Although this setting is simpler than the vector-input case, it already
captures the essential structural features of IC--MLPs and serves as a
transparent model for the general theory developed later in this paper.

For completeness, we describe explicitly the
network architectures with zero, one, two, three, and $L$ hidden layers. This
layer-by-layer presentation is intended to clarify the role of direct input
connections at each depth and to highlight the structural differences between
IC--MLPs and classical feedforward networks.

\medskip

\textbf{Zero Hidden Layer (Depth 0).}
We begin with the simplest possible architecture, namely a network with no
hidden layers. In this case, the input signal $x$ is passed directly to the
output neuron, which applies an affine transformation. The resulting network
function has the form
\[
H_0(x)=c\,x+d,
\]
where $c,d\in\mathbb{R}$ are parameters of the output neuron. Such a network
computes only linear functions of the input and is therefore incapable of
approximating nonlinear functions. This trivial case is included
mainly for completeness and as a baseline for comparison with deeper networks.

\medskip

\textbf{Single Hidden Layer (Depth 1).}
We next consider an IC--MLP with a single hidden layer consisting of $N$
neurons. The input neuron outputs the raw input value $x$. Each hidden neuron
receives this input, applies an affine transformation, and then applies the
activation function $\sigma$. More precisely, the output of the $i$-th hidden
neuron is given by
\[
h_{1,i}=\sigma(a_{1,i}x+b_{1,i}), \qquad i=1,\dots,N,
\]
where $a_{1,i},b_{1,i}\in\mathbb{R}$ are parameters of that neuron.

The output neuron then combines all hidden-layer outputs and, in addition,
receives the raw input $x$ directly through an affine connection. The network
output is therefore
\[
H_1(x)=\sum_{i=1}^N v_i\,h_{1,i}+c\,x+d
=\sum_{i=1}^N v_i\,\sigma(a_{1,i}x+b_{1,i})+c\,x+d,
\]
where $v_i,c,d\in\mathbb{R}$ are output-layer parameters. Compared to a
classical single-hidden-layer MLP, the distinctive feature here is the
additional direct affine term $c\,x+d$ at the output.

\medskip

\textbf{Two Hidden Layers (Depth 2).}
We now move to an IC--MLP with two hidden layers. The first hidden layer has
$N_1$ neurons and is identical in structure to the single-hidden-layer case,
with outputs
\[
h_{1,i}=\sigma(a_{1,i}x+b_{1,i}), \qquad i=1,\dots,N_1.
\]
The second hidden layer consists of $N_2$ neurons. Each neuron in this layer
receives all outputs from the first hidden layer and, simultaneously, the raw
input $x$ through an affine connection. The output of the $j$-th neuron in the
second hidden layer is therefore
\[
h_{2,j}
=
\sigma\Big(
\sum_{i=1}^{N_1} w_{2,ji}\,h_{1,i}+a_{2,j}x+b_{2,j}
\Big),
\qquad j=1,\dots,N_2,
\]
where $w_{2,ji},a_{2,j},b_{2,j}\in\mathbb{R}$ are parameters.

As before, the output neuron combines all second-layer outputs and includes a
direct affine dependence on $x$. The overall network function is thus
\[
H_2(x)=
\sum_{j=1}^{N_2} v_j\,
\sigma\Big(
\sum_{i=1}^{N_1} w_{2,ji}\,
\sigma(a_{1,i}x+b_{1,i})
+a_{2,j}x+b_{2,j}
\Big)
+c\,x+d.
\]
This expression illustrates the defining feature of IC--MLPs:
at each depth, the outputs of the previous layer are combined 
with a direct affine term depending on the input.

\medskip

\textbf{Three Hidden Layers (Depth 3).}
The same construction extends naturally to deeper architectures. This case is 
included solely to make the recursive structure explicit. For three
hidden layers with $N_1$, $N_2$, and $N_3$ neurons, respectively, the first two
layers are defined as above, while the third hidden layer produces outputs
\[
h_{3,k}
=
\sigma\Big(
\sum_{j=1}^{N_2} w_{3,kj}\,h_{2,j}+a_{3,k}x+b_{3,k}
\Big),
\qquad k=1,\dots,N_3.
\]
The output neuron again aggregates all third-layer outputs together with a
direct affine contribution from $x$, yielding
\[
H_3(x)=
\sum_{k=1}^{N_3} v_k\,
\sigma\Big(
\sum_{j=1}^{N_2} w_{3,kj}\,
\sigma\Big(
\sum_{i=1}^{N_1} w_{2,ji}\,
\sigma(a_{1,i}x+b_{1,i})
+a_{2,j}x+b_{2,j}
\Big)
+a_{3,k}x+b_{3,k}
\Big)
+c\,x+d.
\]

\medskip

\textbf{General $L$-Hidden-Layer Scalar IC--MLP.}
In full generality, an IC--MLP with $L$ hidden layers is defined recursively.
If the $\ell$-th hidden layer contains $N_\ell$ neurons, then the output of the
$j$-th neuron in that layer is given by
\[
h_{\ell,j}
=
\sigma\Big(
\sum_{i=1}^{N_{\ell-1}} w_{\ell,ji}\,h_{\ell-1,i}
+a_{\ell,j}x+b_{\ell,j}
\Big),
\qquad \ell=1,\dots,L.
\]
Here, $h_{0,1}=x$ is understood as the output of the input neuron. The final
output of the network is computed by
\[
H_L(x)=\sum_{j=1}^{N_L} v_j\,h_{L,j}+c\,x+d.
\]
Equivalently, the function realized by an $L$-hidden-layer scalar IC--MLP can
be written in the nested form
\[
H_L(x)=
\sum_{j=1}^{N_L} v_j\,
\sigma\Big(
\sum_{i=1}^{N_{L-1}} w_{L,ji}\,
\sigma\big(\cdots \sigma(a_{1,i}x+b_{1,i})\cdots\big)
+a_{L,j}x+b_{L,j}
\Big)
+c\,x+d.
\]

This architecture is illustrated schematically in Figure~\ref{fig:ic-mlp-scalar}.

\medskip

\noindent\textbf{Remark.}
A standard feedforward scalar MLP is obtained as a special case of the scalar
IC--MLP architecture by suppressing all direct input connections beyond the
first hidden layer. In the present notation, this corresponds to imposing
\[
a_{\ell,j}=0
\quad \text{for all hidden layers } \ell\ge 2 \text{ and all neurons } j,
\quad\text{and}\quad c=0.
\]
Under this restriction, only the first hidden layer receives the raw input,
and the architecture reduces exactly to a classical scalar MLP. Consequently,
every function realizable by a standard scalar MLP is also realizable by a
scalar IC--MLP, while the converse does not hold in general.

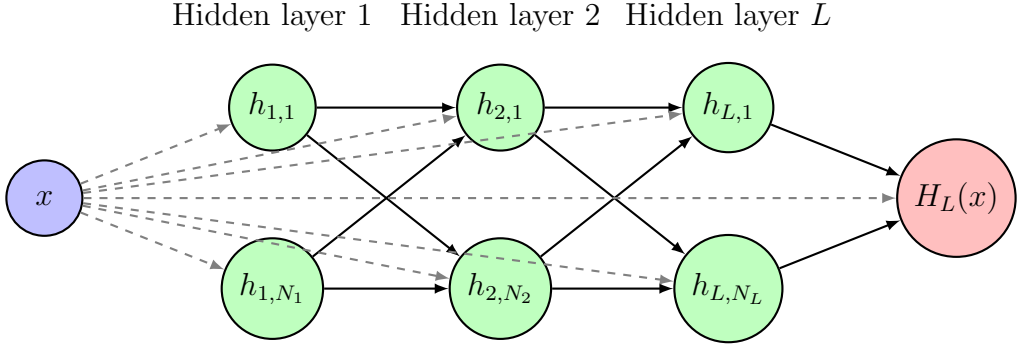
\begin{figure}[ht]
\centering
\begin{tikzpicture}[
    neuron/.style={circle, draw, thick, minimum size=10mm},
    input/.style={neuron, fill=blue!25},
    hidden/.style={neuron, fill=green!25},
    output/.style={neuron, fill=red!25},
    conn/.style={->, thick},
    skip/.style={->, thick, dashed, gray},
    >=latex
]

\node[input] (x) at (0,0) {$x$};

\node[hidden] (h11) at (3, 1.2) {$h_{1,1}$};
\node[hidden] (h12) at (3,-1.2) {$h_{1,N_1}$};

\node[hidden] (h21) at (6, 1.2) {$h_{2,1}$};
\node[hidden] (h22) at (6,-1.2) {$h_{2,N_2}$};

\node[hidden] (hL1) at (9, 1.2) {$h_{L,1}$};
\node[hidden] (hL2) at (9,-1.2) {$h_{L,N_L}$};

\node[output] (y) at (12,0) {$H_L(x)$};

\foreach \i in {h11,h12}
  \foreach \j in {h21,h22}
    \draw[conn] (\i) -- (\j);

\foreach \i in {h21,h22}
  \foreach \j in {hL1,hL2}
    \draw[conn] (\i) -- (\j);

\foreach \i in {hL1,hL2}
  \draw[conn] (\i) -- (y);

\foreach \h in {h11,h12,h21,h22,hL1,hL2}
  \draw[skip] (x) -- (\h);

\draw[skip] (x) -- (y);

\node at (3,2.4) {Hidden layer 1};
\node at (6,2.4) {Hidden layer 2};
\node at (9,2.4) {Hidden layer $L$};

\end{tikzpicture}

\caption{
Scalar Input-Connected Multilayer Perceptron (IC--MLP).
The scalar input $x$ is fed directly to every hidden layer and to the output
neuron (dashed arrows), in addition to the standard feedforward connections
between successive layers (solid arrows).
}
\label{fig:ic-mlp-scalar}
\end{figure}

\section{Universal Approximation Theorem for Univariate Functions}

In this section, we establish a universal approximation theorem for IC--MLPs
in the univariate setting. The theorem below characterizes precisely when
deep input connected neural networks are capable of approximating arbitrary continuous
functions on compact sets of the real line.

\begin{theorem}
Let $\sigma:\mathbb{R}\to\mathbb{R}$ be continuous. The following statements are equivalent:
\begin{enumerate}
    \item $\sigma$ is nonlinear.
    \item For every closed interval $[\alpha,\beta] \subset \mathbb{R}$, every continuous 
    function $f:[\alpha,\beta]\to\mathbb{R}$, and every $\varepsilon>0$, there exist 
    an integer $L\ge 1$ and parameters of an $L$-hidden-layer IC--MLP such that 
    the corresponding network function $H_L$ satisfies
    \[
    \sup_{x\in [\alpha,\beta]} |f(x)-H_L(x)| < \varepsilon.
    \]
\end{enumerate}

In other words, deep IC--MLP networks can approximate continuous functions on compact 
sets of $\mathbb{R}$ if and only if the activation function $\sigma$ is nonlinear.
\end{theorem}

\begin{proof}
\noindent\textbf{Necessity.}
Assume that $\sigma$ is linear, that is,
\[
\sigma(t)=At+B.
\]
We show that every IC--MLP computes a linear function of the input $x$.

The input neuron outputs $x$, which is linear. Suppose that all neurons in some layer 
compute linear functions of $x$. Each neuron in the subsequent layer evaluates
\[
\sigma\!\left(\sum_k w_k h_k(x) + c x + d\right),
\]
where $h_k(x)$ are linear. The argument of $\sigma$ is therefore linear in $x$, 
and since $\sigma$ is linear, the output is again linear. By induction on the depth, 
every IC--MLP realizes a linear function. Consequently, no nonlinear continuous function 
can be approximated, and universality fails.

\medskip

\noindent\textbf{Sufficiency.}
Assume now that $\sigma$ is continuous and nonlinear.

Let $\mathcal{F}$ denote the set of all functions $\mathbb{R}\to\mathbb{R}$ 
realizable by finite IC--MLPs. By construction, $\mathcal{F}$ contains linear functions, 
and is closed under linear combinations, affine changes of variables, and superposition with $\sigma$.

Since $\sigma$ is nonlinear on $\mathbb{R}$, the function $\sigma(ax+b)$ is nonlinear
for every $a\neq 0$ and $b\in\mathbb{R}$. Indeed, if $\sigma(ax+b)$ were linear for some
$a\neq 0$, then composing with the inverse affine map $y\mapsto (y-b)/a$ would imply
that $\sigma$ itself is linear on $\mathbb{R}$, a contradiction.

Let $\varphi\in C_c^\infty(\mathbb{R})$ be a nonnegative function with 
$\int_{\mathbb{R}}\varphi(y)\,dy=1$,
and define $\varphi_\varepsilon(t)=\varepsilon^{-1}\varphi(t/\varepsilon)$. Fix some
$a\neq 0$ and $b\in\mathbb R$, and set
\[
S_\varepsilon(x):=(\sigma(ax+b)*\varphi_\varepsilon)(x)
=\int_{\mathbb R}\sigma(a(x-y)+b)\,\varphi_\varepsilon(y)\,dy .
\]
We claim that $S_\varepsilon\in\overline{\mathcal F}$. Indeed, each function 
$\sigma(a(x-y)+b)$ belongs to $\mathcal{F}$, and 
finite linear combinations of such functions are realizable by IC--MLPs. 
Moreover, since $\varphi_\varepsilon$ has compact support, the integral defining 
$S_\varepsilon$ can be approximated uniformly on any compact set $K\subset\mathbb{R}$ 
by Riemann sums of the form
\[
S_{\varepsilon,n}(x) = \sum_{k=1}^{n} \lambda_k \, \sigma(a(x-y_k)+b),
\]
where $\lambda_k\in\mathbb{R}$ and $y_k$ are points in the support of $\varphi_\varepsilon$. 
Each $S_{\varepsilon,n}$ is a finite linear combination of functions from $\mathcal{F}$, 
and 
\[
\sup_{x\in K} |S_{\varepsilon,n}(x) - S_\varepsilon(x)| \longrightarrow 0 \quad \text{as } n\to\infty.
\]

Thus $S_\varepsilon$ is the uniform limit on compact sets of such finite linear combinations, 
and therefore lies in the closure of $\mathcal{F}$ in $C(\mathbb{R})$ with the topology of 
uniform convergence on compact sets.

Moreover, $S_\varepsilon \in C^\infty(\mathbb{R})$ because it is the convolution of the 
function $\sigma(ax+b)$ with the smooth mollifier $\varphi_\varepsilon$. By the standard properties of mollifiers, 
\[
S_\varepsilon(x) \to \sigma(ax+b)
\]
uniformly on compact sets as $\varepsilon \to 0$.

For sufficiently small $\varepsilon$, the function $S_\varepsilon$ is nonlinear. 
Indeed, suppose by contradiction that there exists a sequence $\varepsilon_n \to 0$ 
such that $S_{\varepsilon_n}$ is linear for all $n$. Then the uniform limit on compact 
sets of $S_{\varepsilon_n}$ would be linear, which would imply that $\sigma(ax+b)$ is 
linear on $\mathbb{R}$. This contradicts the nonlinearity of $\sigma$. Hence, there exists 
$\varepsilon_0>0$ such that $S_\varepsilon$ is nonlinear for all $0<\varepsilon<\varepsilon_0$.

Fix one $\varepsilon>0$ such that $S:=S_\varepsilon$ is nonlinear. Then there
exists $x_0\in\mathbb R$ with $S''(x_0)\neq 0$. For $\delta>0$ define the scaled symmetric
second difference
\[
T_\delta(x):=\frac{S(x_0+\delta x)+S(x_0-\delta x)-2S(x_0)}{\delta^2\,S''(x_0)}.
\]

Let us first prove that the function $x\mapsto S(x_0+\delta x)$ belongs to
$\overline{\mathcal F}$. Since
$S\in\overline{\mathcal F}$, there exists a sequence $(F_n)_{n\ge1}\subset\mathcal F$
such that $F_n\to S$ uniformly on every compact subset of $\mathbb R$. Define
\[
G_n(x):=F_n(x_0+\delta x),\qquad x\in\mathbb R.
\]
Each function $G_n$ belongs to $\mathcal F$, because if $F_n$ is realized by an IC--MLP
with input $x$, then replacing the input by the affine map $x\mapsto x_0+\delta x$
(i.e.\ feeding $\delta x$ and adding the bias $x_0$ through the direct input connection)
yields an IC--MLP realizing $x\mapsto F_n(x_0+\delta x)$. 

Let $K\subset\mathbb R$ be compact and set
\[
K':=x_0+\delta K=\{x_0+\delta x:\ x\in K\},
\]
which is also compact. Then
\[
\sup_{x\in K}\bigl|G_n(x)-S(x_0+\delta x)\bigr|
=\sup_{x\in K}\bigl|F_n(x_0+\delta x)-S(x_0+\delta x)\bigr|
\le \sup_{y\in K'}|F_n(y)-S(y)|.
\]
Since $F_n\to S$ uniformly on $K'$, the right-hand side tends to $0$ as $n\to\infty$.
Thus $G_n\to S(x_0+\delta x)$ uniformly on $K$. Since $K$ was arbitrary, this shows that
$S(x_0+\delta x)\in\overline{\mathcal F}$. The same argument applies to
$S(x_0-\delta x)$.

We have therefore proved that both $S(x_0+\delta x)$ and $S(x_0-\delta x)$ belong to
$\overline{\mathcal F}$. Since $\overline{\mathcal F}$ is closed under linear combinations and contains
constants, it follows that $T_\delta\in\overline{\mathcal F}$ for every $\delta>0$.

We claim that $T_\delta\to x^2$ uniformly on compact sets as $\delta\to 0$.
Fix $M>0$ and assume $|x|\le M$. By Taylor's theorem with remainder (using $S\in C^3$),
\[
S(x_0\pm \delta x)
=S(x_0)\pm \delta x\,S'(x_0)+\frac{\delta^2x^2}{2}\,S''(x_0)
\pm \frac{\delta^3x^3}{6}\,S^{(3)}(\xi_\pm),
\]
for some $\xi_\pm$ between $x_0$ and $x_0\pm \delta x$. Adding the $+$ and $-$ expansions
cancels the linear terms and yields
\[
S(x_0+\delta x)+S(x_0-\delta x)-2S(x_0)
=\delta^2x^2 S''(x_0)+R_\delta(x),
\]
where
\[
|R_\delta(x)|
\le \frac{\delta^3|x|^3}{6}\bigl(|S^{(3)}(\xi_+)|+|S^{(3)}(\xi_-)|\bigr)
\le \frac{\delta^3M^3}{3}\sup_{|t-x_0|\le \delta M}|S^{(3)}(t)|.
\]
Therefore
\[
\sup_{|x|\le M}|T_\delta(x)-x^2|
\le \frac{\delta M^3}{3|S''(x_0)|}\sup_{|t-x_0|\le \delta M}|S^{(3)}(t)|
\longrightarrow 0
\qquad(\delta\to 0),
\]
which proves $T_\delta\to x^2$ uniformly on $[-M,M]$. Hence $x^2\in\overline{\mathcal F}$.

Since the function $x^2$ belongs to the closure of $\mathcal{F}$, for any $w$ in the closure, the composition
\[
w^2 = x^2 \circ w
\]
also belongs to the closure. Indeed, for $n\in\mathbb N$ set $K_n:=[-n,n]$.
Since $w$ is continuous, the set $w(K_n)$ is compact.
Since $x^2\in\overline{\mathcal F}$, we can choose $P_n\in\mathcal F$ such that
\[
\sup_{y\in w(K_n)} |P_n(y)-y^2|<\frac{1}{2n}.
\]
By continuity of $P_n$, the function $P_n$ is uniformly continuous on the compact set
\[
J_n:=\{y\in\mathbb R:\ \mathrm{dist}(y,w(K_n))\le 1\}.
\]
Choose $\delta_n>0$ with $\delta_n\le 1$ so that
\[
|P_n(y)-P_n(z)|<\frac{1}{2n}
\quad\text{whenever } y,z\in J_n\text{ and } |y-z|<\delta_n.
\]
Since $w\in\overline{\mathcal F}$, choose $w_n\in\mathcal F$ such that
\[
\sup_{x\in K_n} |w_n(x)-w(x)|<\delta_n.
\]
Define
\[
Q_n:=P_n\circ w_n.
\]
Then $Q_n\in\mathcal F$. Moreover, for $x\in K_n$ we have $w(x)\in w(K_n)\subset J_n$ and
\[
\mathrm{dist}\bigl(w_n(x),w(K_n)\bigr)\le |w_n(x)-w(x)|<\delta_n\le 1,
\]
hence $w_n(x)\in J_n$. Therefore,
\[
\begin{aligned}
|Q_n(x)-w(x)^2|
&=|P_n(w_n(x))-w(x)^2|\\
&\le |P_n(w_n(x))-P_n(w(x))|
   +|P_n(w(x))-w(x)^2|\\
&<\frac{1}{2n}+\frac{1}{2n}
=\frac{1}{n},
\end{aligned}
\]
and thus
\[
\sup_{x\in K_n}|Q_n(x)-w(x)^2|<\frac{1}{n}.
\]

Finally, let $K\subset\mathbb R$ be any compact set.
Then $K\subset K_N$ for some $N\in\mathbb N$.
For all $n\ge N$ we have $K\subset K_n$, so
\[
\sup_{x\in K}|Q_n(x)-w(x)^2|
\le \sup_{x\in K_n}|Q_n(x)-w(x)^2|
<\frac{1}{n}\xrightarrow[n\to\infty]{}0.
\]
Thus $Q_n\to w^2$ uniformly on $K$.
Since $K$ was arbitrary, $Q_n\to w^2$ uniformly on compact sets, i.e.\ $w^2\in\overline{\mathcal F}$.

Next, for any $u,v$ in the closure, their product can be expressed as
\[
uv = \frac{1}{2} \Big[ (u+v)^2 - u^2 - v^2 \Big],
\]
showing that the closure is closed under pointwise multiplication. Iterating this argument, 
it follows that all monomials $x^n$ and hence all polynomials in $x$ are contained in the closure of $\mathcal{F}$.

By the Weierstrass approximation theorem, polynomials are dense in $C([\alpha,\beta])$ 
for any closed interval $[\alpha,\beta] \subset \mathbb{R}$. Therefore, for any $f\in C([\alpha,\beta])$ 
and $\varepsilon>0$, there exists an IC--MLP function $H_L$ such that
\[
\sup_{x\in[\alpha,\beta]} |f(x)-H_L(x)| < \varepsilon.
\]

This completes the proof of the theorem.
\end{proof}

\section{Vector-Valued Inputs and Multivariate Universal Approximation}

\subsection{IC--MLP Architecture with Vector Inputs}

The IC--MLP architecture with vector inputs is a natural extension of the
scalar-input IC--MLPs considered earlier in this paper. Instead of a single real-valued
input, the network now processes an input vector
\[
\mathbf{x}=(x_1,\dots,x_n)\in\mathbb{R}^n.
\]
As in the scalar case, the architecture is built around a fixed continuous
activation function $\sigma:\mathbb{R}\to\mathbb{R}$. The defining feature
of IC--MLPs is that direct affine connections from the raw input are allowed
not only at the first hidden layer, but at every hidden layer of the network.

The input layer consists of $n$ neurons whose sole purpose is to output the
coordinate functions of the input vector. Thus, the signals made available
to the rest of the network are precisely
\[
\mathbf{x}\mapsto x_1,\dots,x_n.
\]
These coordinate functions serve as the basic building blocks from which all 
subsequent computations are formed.

The network then proceeds through a sequence of hidden layers. Suppose that
the $\ell$-th hidden layer contains $N_\ell$ neurons. Each neuron in this
layer receives two types of input. First, it receives all outputs from the
previous layer, denoted by
$h_{\ell-1,1},\dots,h_{\ell-1,N_{\ell-1}}$.
Second, it also receives the raw input vector $\mathbf{x}$ through an affine
mapping. More precisely, the $j$-th neuron in the $\ell$-th hidden layer
computes
\[
h_{\ell,j}
=
\sigma\!\Big(
\sum_{i=1}^{N_{\ell-1}} w_{\ell,ji} \, h_{\ell-1,i}
+ \langle \mathbf{a}_{\ell,j}, \mathbf{x} \rangle
+ b_{\ell,j}
\Big),
\]
where the coefficients $w_{\ell,ji}$ describe the inter-layer connections,
while $\mathbf{a}_{\ell,j}\in\mathbb{R}^n$ and $b_{\ell,j}\in\mathbb{R}$
parameterize the direct affine dependence on the input vector. 
For the first hidden layer, this formula is understood by setting
\(h_{0,i}=x_i\) for \(i=1,\dots,n\), so that the summation term represents an
affine function of the input vector~\(\mathbf{x}\).

After passing through $L$ hidden layers, the network produces its final
output via a single output neuron. This neuron aggregates all signals from
the last hidden layer and, in addition, incorporates a direct affine
dependence on the input vector. The output function computed by the network
is therefore given by
\[
H_L(\mathbf{x})
=
\sum_{j=1}^{N_L} v_j \, h_{L,j}
+ \langle \mathbf{c}, \mathbf{x} \rangle + d,
\]
where $v_j\in\mathbb{R}$ are the output-layer weights and
$\mathbf{c}\in\mathbb{R}^n$, $d\in\mathbb{R}$ determine the affine
contribution of the raw input at the output stage.

To make the structure more explicit, consider first the case of a single
hidden layer. In this situation, the network computes functions of the form
\[
H_1(\mathbf{x})
=
\sum_{j=1}^{N_1} v_j \,
\sigma\big(\langle \mathbf{a}_{1,j}, \mathbf{x} \rangle + b_{1,j}\big)
+ \langle \mathbf{c}, \mathbf{x} \rangle + d.
\]
Each hidden neuron applies the activation function to an affine function of
the input, and the output layer combines these responses together with an
additional affine term.

For two hidden layers, the structure becomes hierarchical. The first hidden
layer produces nonlinear functions of the input, while the second hidden
layer combines these functions with a further direct affine dependence on
$\mathbf{x}$. The resulting network function has the form
\[
H_2(\mathbf{x})
=
\sum_{j=1}^{N_2} v_j \,
\sigma\Big(
\sum_{i=1}^{N_1} w_{2,ji} \,
\sigma(\langle \mathbf{a}_{1,i}, \mathbf{x} \rangle + b_{1,i})
+ \langle \mathbf{a}_{2,j}, \mathbf{x} \rangle + b_{2,j}
\Big)
+ \langle \mathbf{c}, \mathbf{x} \rangle + d.
\]

In general, for $L$ hidden layers, the IC--MLP computes functions obtained by
iterating this construction. Each hidden layer combines the outputs of the
previous layer with an affine function of the original input vector, applies
the nonlinearity $\sigma$, and forwards the result to the next layer. The
resulting network function can be written as
\[
H_L(\mathbf{x})
=
\sum_{j=1}^{N_L} v_j \,
\sigma\Bigg(
\sum_{i=1}^{N_{L-1}} w_{L,ji} \,
\sigma\Big(\cdots
\sigma\big(\langle \mathbf{a}_{1,i}, \mathbf{x}\rangle + b_{1,i}\big)
\cdots\Big)
+ \langle \mathbf{a}_{L,j}, \mathbf{x} \rangle + b_{L,j}
\Bigg)
+ \langle \mathbf{c}, \mathbf{x} \rangle + d.
\]
This expression makes explicit that, unlike standard feedforward networks,
the raw input vector $\mathbf{x}$ may influence the computation at every
depth of the network.

An illustration of the corresponding vector-input architecture is given in
Figure~\ref{fig:ic-mlp-vector}.

\begin{figure}[ht]
\centering
\begin{tikzpicture}[
    neuron/.style={circle, draw, thick, minimum size=10mm}, 
    input/.style={neuron, fill=blue!25},
    hidden/.style={neuron, fill=green!25},
    output/.style={neuron, fill=red!25},
    conn/.style={->, thick},
    skip/.style={->, thick, dashed, gray},
    layer/.style={font=\small}
]

\node[input] (x1) at (0, 1.6) {$x_1$};
\node[input] (x2) at (0, 0.4) {$x_2$};
\node at (0,-0.5) {$\vdots$};
\node[input] (xn) at (0,-1.7) {$x_n$};

\node[layer] at (0,2.6) {Input};

\node[hidden] (h11) at (3, 2.0) {$h_{1,1}$};
\node[hidden] (h12) at (3, 0.8) {$h_{1,2}$};
\node[hidden] (h13) at (3,-0.4) {$h_{1,3}$};

\node[layer] at (3,2.9) {Hidden layer 1};

\node[hidden] (h21) at (6, 1.3) {$h_{2,1}$};
\node[hidden] (h22) at (6, 0.0) {$h_{2,2}$};

\node[layer] at (6,2.9) {Hidden layer 2};

\node at (8.2,0.7) {$\cdots$};

\node[output] (y) at (10,0.7) {$H_L(\mathbf{x})$};
\node[layer] at (10,2.9) {Output};

\foreach \i in {x1,x2,xn}{
  \foreach \j in {h11,h12,h13}{
    \draw[conn] (\i) -- (\j);
  }
}

\foreach \i in {h11,h12,h13}{
  \foreach \j in {h21,h22}{
    \draw[conn] (\i) -- (\j);
  }
}

\foreach \i in {h21,h22}{
  \draw[conn] (\i) -- (y);
}

\foreach \i in {x1,x2,xn}{
  \draw[skip] (\i) .. controls (4.5,-2.7) and (5.5,-2.7) .. (h21);
  \draw[skip] (\i) .. controls (4.8,-2.7) and (5.8,-2.7) .. (h22);
  \draw[skip] (\i) .. controls (7,-3.2) and (9,-2.2) .. (y);
}

\end{tikzpicture}

\caption{
Input-Connected Multilayer Perceptron (IC--MLP) with vector input.
The input $(x_1,\dots,x_n)$ is fed directly to every hidden layer and to the output
neuron (dashed arrows), in addition to the standard feedforward connections between
successive layers (solid arrows).
}
\label{fig:ic-mlp-vector}
\end{figure}
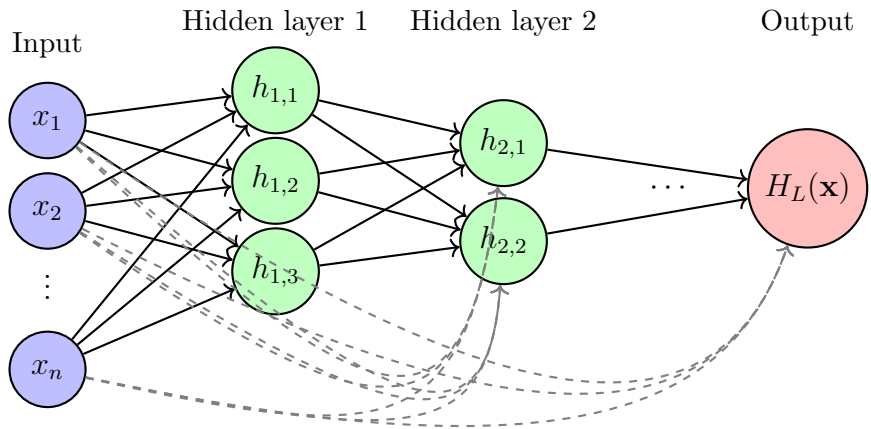

Finally, the relation between IC--MLPs and classical multilayer perceptrons
can be made precise. A standard feedforward MLP is obtained as a special case
of the IC--MLP architecture by suppressing all direct input connections
beyond the first hidden layer. In the present notation, this corresponds to
imposing the constraint
\[
\mathbf{a}_{\ell,j}=\mathbf{0}
\quad \text{for all neurons } j \text{ and all layers } \ell\ge 2,
\quad\text{and}\quad \mathbf{c}=\mathbf{0}.
\]
Under this restriction, only the first hidden layer receives the raw input,
and all deeper layers depend exclusively on the outputs of the preceding
layer, exactly as in a classical MLP. Consequently, every function realizable
by a standard MLP is also realizable by an IC--MLP. However, the converse is
not true in general: by allowing direct affine connections from the input at
every hidden layer, IC--MLPs form a strictly more expressive class of
networks.

\subsection{Universal Approximation Theorem for Multivariate Functions}

We now extend the universal approximation result from the univariate setting
to IC--MLPs with vector-valued inputs. The theorem below shows that 
the same criterion governs approximation in higher dimensions: deep
IC--MLP networks can approximate arbitrary continuous functions on compact
subsets of $\mathbb{R}^n$ if and only if the activation function is nonlinear.
The proof builds on the scalar case and exploits the additional structural
flexibility provided by direct input connections at every hidden layer.

\begin{theorem}
Let $\sigma:\mathbb{R}\to\mathbb{R}$ be continuous.
The following statements are equivalent:
\begin{enumerate}
\item $\sigma$ is nonlinear.
\item For every compact set $K\subset\mathbb{R}^n$, every
$f\in C(K)$, and every $\varepsilon>0$, there exist an integer $L\ge1$
and parameters of an $L$-hidden-layer IC--MLP such that
\[
\sup_{x\in K} |f(x)-H_L(x)| < \varepsilon.
\]
\end{enumerate}

That is, nonlinearity of the activation function is a necessary and sufficient
condition for IC--MLPs to approximate every continuous multivariate function on
compact sets.
\end{theorem}

\begin{proof}
\noindent\textbf{Necessity.}
Assume that $\sigma$ is linear, that is,
\[
\sigma(t)=A t + B
\]
for some constants $A,B\in\mathbb{R}$.

We show that every IC--MLP with vector input $\mathbf{x}\in\mathbb{R}^n$ computes an affine function
of $\mathbf{x}$.

The input layer outputs the coordinate functions
\[
\mathbf{x}\mapsto x_1,\dots,x_n,
\]
which are linear. Suppose that all neurons in some hidden layer $\ell-1$ compute affine functions of
$\mathbf{x}$.
A neuron $j$ in the $\ell$-th hidden layer computes
\[
h_{\ell,j}(\mathbf{x})
=
\sigma\!\Big(
\sum_{i=1}^{N_{\ell-1}} w_{\ell,ji} h_{\ell-1,i}(\mathbf{x})
+ \langle \mathbf{a}_{\ell,j}, \mathbf{x} \rangle
+ b_{\ell,j}
\Big).
\]
By the induction hypothesis, each $h_{\ell-1,i}(\mathbf{x})$ is affine in $\mathbf{x}$.
Therefore, the argument of $\sigma$ is an affine function of $\mathbf{x}$.
Since $\sigma$ itself is linear, the output $h_{\ell,j}(\mathbf{x})$ is again affine in $\mathbf{x}$.

By induction over the depth, every hidden neuron computes an affine function of $\mathbf{x}$.
The output neuron computes
\[
H_L(\mathbf{x})
=
\sum_{j=1}^{N_L} v_j h_{L,j}(\mathbf{x})
+ \langle \mathbf{c}, \mathbf{x} \rangle
+ d,
\]
which is affine in $\mathbf{x}$. Hence every IC--MLP realizes an affine function
$\mathbb{R}^n\to\mathbb{R}$.
Such functions cannot approximate arbitrary nonlinear continuous functions on compact sets,
so universality fails.

\medskip

\noindent\textbf{Sufficiency.}
Assume now that $\sigma$ is continuous and nonlinear.

We denote by $\mathcal{F}_n$ the set of all functions
$\mathbb{R}^n\to\mathbb{R}$ realizable by finite IC--MLPs, and by
$\overline{\mathcal{F}_n}$ its closure in the topology of uniform
convergence on compact sets.

First, constant functions belong to $\mathcal{F}_n$.
Indeed, choosing $v_j=0$ for all $j$ and $\mathbf{c}=\mathbf{0}$ in the
output layer yields $H_L(\mathbf{x})\equiv d$.
Moreover, every affine function
\[
\mathbf{x}\mapsto \langle \mathbf{a},\mathbf{x}\rangle + b
\]
belongs to $\mathcal{F}_n$ by taking a network with no hidden layers and
setting $\mathbf{c}=\mathbf{a}$ and $d=b$.
In particular, each coordinate projection $\mathbf{x}\mapsto x_i$
belongs to $\mathcal{F}_n$.

We now verify that $\mathcal{F}_n$ is closed under linear combinations.
Fix $L\ge 0$ and consider the class $\mathcal{F}_{n,L}$ of functions realized
by IC--MLPs with exactly $L$ hidden layers. Since 
the output neuron computes the sum of a linear combination of the outputs
of the last hidden layer and an affine function of the input,
$\mathcal{F}_{n,L}$ is a linear space.

Furthermore, if $M<L$, then $\mathcal F_{n,M}\subset\mathcal F_{n,L}$.
Indeed, the defining formula of an $L$-hidden-layer IC--MLP contains that of
an $M$-hidden-layer IC--MLP as a special case, obtained by setting the
inter-layer weights of the additional hidden layers equal to zero.
Consequently, for any $f,g\in\mathcal{F}_n$ there exists $L$ such that both
$f$ and $g$ belong to $\mathcal{F}_{n,L}$, and hence any linear combination
$\alpha f+\beta g$ is again realized by an IC--MLP with $L$ hidden layers. 
Therefore, $\mathcal{F}_n$ is closed under linear combinations.

We next explain the stability of $\mathcal{F}_n$ under superposition with scalar
IC--MLPs. 

Let $f\in\mathcal{F}_n$ and let $g:\mathbb{R}\to\mathbb{R}$ belong to
$\mathcal{F}_1$ or to its closure $\overline{\mathcal{F}_1}$.
If $g\in\mathcal{F}_1$, then $g$ is realized by a finite IC--MLP with
one-dimensional input. By feeding the scalar output of the IC--MLP realizing $f$
into the input layer of the network realizing $g$, we obtain a finite IC--MLP that
realizes the composition
\[
\mathbf{x}\longmapsto g\big(f(\mathbf{x})\big),
\]
and hence $g\circ f$ belongs to $\mathcal{F}_n$.

If instead $g\in\overline{\mathcal{F}_1}$, let $K\subset\mathbb{R}^n$ be compact.
Since $f$ is continuous, the image $f(K)\subset\mathbb{R}$ is compact.
There exists a sequence $\{g_m\}\subset\mathcal{F}_1$ such that
\[
\sup_{y\in f(K)} |g_m(y)-g(y)| \longrightarrow 0
\quad \text{as } m\to\infty.
\]
For each $m$, the function $\mathbf{x}\mapsto g_m(f(\mathbf{x}))$ belongs to
$\mathcal{F}_n$, and
\[
\sup_{\mathbf{x}\in K}
\big|
g_m(f(\mathbf{x}))-g(f(\mathbf{x}))
\big|
=
\sup_{y\in f(K)} |g_m(y)-g(y)|
\longrightarrow 0
\quad \text{as } m\to\infty.
\]
Hence $g\circ f\in\overline{\mathcal{F}_n}$.

We now use the scalar universality result.
By Theorem~1, since $\sigma$ is continuous and nonlinear, the closure of
$\mathcal{F}_1$ coincides with $C(\mathbb{R})$.
In particular, the scalar function
\[
t\longmapsto t^2
\]
belongs to $\overline{\mathcal{F}_1}$.
Let $\{s_m\}\subset\mathcal{F}_1$ be a sequence such that
$s_m(t)\to t^2$ uniformly on compact subsets of $\mathbb{R}$.

Let $f\in\overline{\mathcal{F}_n}$ and let $K\subset\mathbb{R}^n$ be
compact.
By definition of the closure, there exists a sequence
$\{f_m\}\subset\mathcal{F}_n$ such that $f_m\to f$ uniformly on $K$.
By the superposition property established above, the functions
$s_m(f_m(\mathbf{x}))$ belong to $\mathcal{F}_n$.
Passing to the limit yields $f^2\in\overline{\mathcal{F}_n}$.
In particular, since each coordinate projection
$\mathbf{x}\mapsto x_i$ belongs to $\mathcal{F}_n$, the function
\[
\mathbf{x}\mapsto x_i^2
\]
lies in $\overline{\mathcal{F}_n}$.

Let now $u,v\in\overline{\mathcal{F}_n}$.
Since $u^2$, $v^2$, and $(u+v)^2$ belong to $\overline{\mathcal{F}_n}$,
the identity
\[
uv=\tfrac12\big[(u+v)^2-u^2-v^2\big]
\]
implies that $uv\in\overline{\mathcal{F}_n}$.
Thus $\overline{\mathcal{F}_n}$ is an algebra containing constants and the
coordinate functions.
By repeated multiplication and linear combination, all monomials
$x_1^{m_1}\cdots x_n^{m_n}$ and hence all polynomials in
$x_1,\dots,x_n$ belong to $\overline{\mathcal{F}_n}$.

Finally, let $K\subset\mathbb{R}^n$ be compact.
By the Stone--Weierstrass theorem, polynomials are dense in $C(K)$.
Since $\overline{\mathcal{F}_n}$ is closed under uniform convergence on
compact sets and contains all polynomials, it follows that for every
$f\in C(K)$ and every $\varepsilon>0$ there exists an IC--MLP function
$H_L$ such that
\[
\sup_{\mathbf{x}\in K}|f(\mathbf{x})-H_L(\mathbf{x})|<\varepsilon.
\]

This completes the proof.
\end{proof}

\medskip

\noindent\textbf{Remark.}
It is useful to make explicit what we mean by a \emph{function realizable by a neural
network} and to highlight a structural distinction between IC--MLPs and 
standard MLPs from an approximation-theoretic point of view.

For each integer $L\ge 0$, let $\mathcal M_{n,L}$ denote the set of functions
$\mathbb R^n\to\mathbb R$ that can be written exactly in the explicit form of an
$L$-hidden-layer standard MLP with activation $\sigma$.
We then define
\[
\mathcal M_n := \bigcup_{L\ge 0} \mathcal M_{n,L},
\]
so that a function $f$ belongs to $\mathcal M_n$ if there exist a finite depth
$L$ and corresponding network parameters such that
\[
f(\mathbf{x}) = H_L(\mathbf{x}) \quad \text{for all } \mathbf{x}\in\mathbb R^n.
\]
Thus, a function is said to be realizable if it coincides pointwise with the
single finite-depth network function $H_L \in \mathcal M_{n,L}$.

With this interpretation, the class of functions produced by IC--MLPs enjoys a closure
property. Since IC--MLPs retain direct access
to the input vector $\mathbf{x}$ at every hidden layer and at the output layer, 
this class is closed under finite linear combinations: if two functions
are realized by IC--MLPs, then any linear combination of them is again realized by a
single IC--MLP of finite depth. This property follows directly from the architecture and
does not depend on the specific choice of activation function.

In contrast, for standard deep MLPs, the class $\mathcal M_n$
defined above is, in general, not closed under linear combinations.
Although individual functions such as $x$ (realizable by a depth--$0$
network) and $\sigma(x)$ (realizable by a depth--$1$ network) both
belong to $\mathcal M_n$, their sum $x+\sigma(x)$ need not be
representable by any single standard MLP of finite depth.

The obstruction is structural. If $f\in\mathcal M_{n,L_f}$ and
$g\in\mathcal M_{n,L_g}$, there need not exist a depth $L$ such that
both $f$ and $g$ belong to $\mathcal M_{n,L}$. In particular,
affine functions are realizable only at depth $0$, and there is no
general mechanism for lifting such functions to networks with positive
depth. Consequently, functions realized at different depths cannot 
always be represented by a network of common depth, and closure under
linear combinations fails.

From an approximation-theoretic perspective, the structural difference between
MLPs and IC--MLPs is significant. The closure under linear combinations enjoyed by
IC--MLPs allows one to construct rich classes of functions in a direct and
transparent way. It also explains why IC--MLPs admit a simple and complete
universality criterion, whereas universality results for standard MLPs typically
require more intricate arguments and/or stronger assumptions on the activation
function.

\section{Discussion and Conclusions}

In the IC--MLP framework, each neuron in a hidden layer receives signals from all
neurons in the previous layer together with a direct affine contribution from
the raw input. This input-connectedness across depth leads to a simple
recursive description of the network function. From an architectural viewpoint, this design
strictly generalizes classical feedforward multilayer perceptrons, which are
recovered by suppressing all direct input connections beyond the first hidden
layer.

Unlike architectures in which direct input connections are introduced to meet
additional structural constraints (for example, convexity requirements in
input-convex neural networks), the IC--MLP employs input-connectedness as a
purely expressive mechanism.
No restrictions are imposed on the signs of the weights,
and no convexity assumptions are made. As a consequence, the zero-hidden-layer
case reduces to a linear model, while networks with one or more hidden layers
and a nonlinear activation function exhibit genuinely nonlinear behavior. The
recursive formulas for $H_L$ precisely characterize the class of functions
realized by the architecture and make explicit how the raw input influences the
computation at every depth.

From the approximation-theoretic viewpoint, the main contribution of this work
is to identify a minimal condition on the activation function that characterizes
universality for IC--MLPs. We first establish this result in the scalar-input
setting, where the structure of the network can be analyzed more easily, and
then extend it to multivariate inputs. We show that, for input-connected
architectures, nonlinearity of the activation function
is both necessary and sufficient for universal approximation. No additional
assumptions such as smoothness, monotonicity, or special structural properties
of the activation are required.

Finally, the present work leaves open several natural questions that merit
further investigation. While we establish a complete universality criterion for
IC--MLPs, our results are qualitative in nature and do not address quantitative
aspects of approximation. In particular, problems concerning approximation
rates, the role of depth and width in controlling the approximation error, and
the number of neurons required to achieve a prescribed accuracy remain open.
Another natural direction is a systematic comparison of approximation
efficiency between IC--MLPs and classical deep architectures.

\end{document}